%% file: root.tex
\documentclass[letterpaper, 10 pt, conference]{ieeeconf}  

\IEEEoverridecommandlockouts                              

\input{preamble} 

\makeatother

\DeclareCaptionFont{mysize}{\fontsize{8}{9.6}\selectfont}
\title{\LARGE \bf Safe Model Predictive Diffusion with Shielding} 

\author{Taekyung Kim$^{1}$, Keyvan Majd$^{2}$, Hideki Okamoto$^{2}$, Bardh Hoxha$^{2}$, Dimitra Panagou$^{1,3}$, Georgios Fainekos$^{2}$
\thanks{This research was performed while Taekyung Kim was an intern at Toyota Motor North America, Research \& Development.}
\thanks{$^{1}$Department of Robotics, $^{3}$Department of Aerospace Engineering, University of Michigan, Ann Arbor, MI, 48109, USA {\tt\footnotesize taekyung@umich.edu, dpanagou@umich.edu} } 
\thanks{$^{2}$Toyota Motor North America, Research \& Development, Ann Arbor, MI, 48105, USA {\tt\footnotesize <first\_name.last\_name>@toyota.com} } %
}

\begin{document}
\maketitle
\thispagestyle{empty}
\pagestyle{empty}

\begin{abstract}
Generating safe, kinodynamically feasible, and optimal trajectories for complex robotic systems is a central challenge in robotics. This paper presents Safe Model Predictive Diffusion (Safe MPD), a training-free diffusion planner that unifies a model-based diffusion framework with a safety shield to generate trajectories that are both kinodynamically feasible and safe by construction. By enforcing feasibility and safety on all samples during the denoising process, our method avoids the common pitfalls of post-processing corrections, such as computational intractability and loss of feasibility. We validate our approach on challenging non-convex planning problems, including kinematic and acceleration-controlled tractor-trailer systems. The results show that it substantially outperforms existing safety strategies in success rate and safety, while achieving sub-second computation times. \href{https://www.taekyung.me/safe-mpd}{\textcolor{red}{[Project Page]}}\footnote{Project page: \href{https://www.taekyung.me/safe-mpd}{https://www.taekyung.me/safe-mpd}} \href{https://github.com/cps-atlas/safe-mpd}{\textcolor{red}{[Code]}} \href{https://youtu.be/DQBeybU7EYI}{\textcolor{red}{[Video]}}
\end{abstract}


\section{INTRODUCTION}
\input{_I.Introduction/intro}

\section{PRELIMINARIES \label{sec:preliminaries}}
\subsection{Problem Formulation \label{subsec:problem}}
\input{_II.Preliminaries/a_problem}

\input{_II.Preliminaries/b_diffusion}
\subsection{Model-Based Diffusion\label{subsec:diffusion}}

\input{_II.Preliminaries/c_mbd}

\section{SAFE MODEL PREDICTIVE DIFFUSION \label{sec:method}}
\input{_III.Methodology/_intro}
\subsection{Model Predictive Diffusion\label{subsec:mpd}}

\input{_III.Methodology/a_mpd}
\subsection{Shielded Rollout\label{subsec:shielding}}
\input{_III.Methodology/b_shielding}
\subsection{Main Algorithm\label{subsec:main}}
\input{_III.Methodology/c_safe}

\section{RESULTS \label{sec:results}}
\input{_V.Experiments/_intro}
\subsection{Dynamical Systems and Environments \label{subsec:dynamics}}

\input{_V.Experiments/a_dynamics}
\subsection{Experimental Setup \& Baseline Methods \label{subsec:setup}}
\input{_V.Experiments/b_baseline}

\subsection{Experimental Results \label{subsec:exp_results}}
\input{_V.Experiments/d_result}

\section{CONCLUSION \label{sec:conclusion}}
\input{_VI.Conclusion/conclusion}

\addtolength{\textheight}{0 cm}   




\bibliographystyle{IEEEtran}
\typeout{}
\bibliography{references.bib}

\end{document}

%% file: preamble.tex
\usepackage{graphics} 
\usepackage{subfig} 
\usepackage{wrapfig}

\usepackage{amsthm} 
\usepackage{amsmath,amssymb}
\usepackage{graphicx}
\usepackage{tabularx}
\usepackage{multirow, makecell}
\usepackage{diagbox}
\usepackage{slashbox}
\usepackage{rotating}
\usepackage{cite}

\usepackage{url}
\usepackage{bm}
\usepackage[table]{xcolor}
\usepackage{hyperref}
\usepackage{siunitx} 
\usepackage{booktabs}
\usepackage{cleveref}
\usepackage{mathtools} 
\usepackage{lipsum} 

\let\labelindent\relax 
\usepackage{enumitem} 

\usepackage[ruled,vlined]{algorithm2e}

\usepackage{pifont} 
\newcommand{\cmark}{\ding{51}}
\newcommand{\xmark}{\ding{55}}

\usepackage[nolist, nohyperlinks]{acronym}
\acrodef{MPC}[MPC]{Model Predictive Control}
\acrodef{QP}[QP]{Quadratic Program}
\acrodef{CBF}[CBF]{Control Barrier Function}

\newcommand{\vx}{{\boldsymbol x}}
\newcommand{\vu}{{\boldsymbol u}}

\newcommand{\StateSpace}{\mathcal{X}}
\newcommand{\ControlSpace}{\mathcal{U}}
\newcommand{\RealSpace}{\mathbb{R}}

\newcommand{\Rn}{\mathbb{R}^{n}}
\newcommand{\Rm}{\mathbb{R}^{m}}

\newcommand{\Natural}{\mathbb{Z}}
\newcommand{\calC}{\mathcal{C}} 

\newcommand{\calP}{\mathcal{P}} 

\newcommand{\calO}{\mathcal{O}} 
\newcommand{\calR}{\mathcal{R}} 
\newcommand{\calJ}{\mathcal{J}} 
\newcommand{\calS}{\mathcal{S}} 
\newcommand{\calN}{\mathcal{N}} 
\newcommand{\calY}{\mathcal{Y}} 

\newcommand{\vepsilon}{{\bm{\varepsilon}}}

\newcommand{\horizon}{T} 

\newcommand{\terminalcost}{l_{\horizon}} 
\newcommand{\runningcost}{l_{t}} 

\DeclareMathAlphabet{\mathmybb}{U}{bbold}{m}{n}
\newcommand{\IndicatorFunction}{\mathmybb{1}} 

\DeclareMathOperator*{\argmin}{arg\,min}

\newtheorem{definition}{Definition}
\newtheorem{theorem}{Theorem}

\newtheorem{remark}{Remark}

\theoremstyle{definition}

\theoremstyle{definition}

\theoremstyle{definition}

\newtheorem{assumption}{Assumption}

\usepackage{fontawesome5}
\newcommand{\shield}{\text{\tiny{\faShield*}}}

\makeatletter
\let\save@mathaccent\mathaccent
\newcommand*\if@single[3]{%
  \setbox0\hbox{${\mathaccent"0362{#1}}^H$}%
  \setbox2\hbox{${\mathaccent"0362{\kern0pt#1}}^H$}%
  \ifdim\ht0=\ht2 #3\else #2\fi
  }
\newcommand*\rel@kern[1]{\kern#1\dimexpr\macc@kerna}
\newcommand*\widebar[1]{\@ifnextchar^{{\wide@bar{#1}{0}}}{\wide@bar{#1}{1}}}
\newcommand*\wide@bar[2]{\if@single{#1}{\wide@bar@{#1}{#2}{1}}{\wide@bar@{#1}{#2}{2}}}
\newcommand*\wide@bar@[3]{%
  \begingroup
  \def\mathaccent##1##2{%
    \let\mathaccent\save@mathaccent
    \if#32 \let\macc@nucleus\first@char \fi
    \setbox\z@\hbox{$\macc@style{\macc@nucleus}_{}$}%
    \setbox\tw@\hbox{$\macc@style{\macc@nucleus}{}_{}$}%
    \dimen@\wd\tw@
    \advance\dimen@-\wd\z@
    \divide\dimen@ 3
    \@tempdima\wd\tw@
    \advance\@tempdima-\scriptspace
    \divide\@tempdima 10
    \advance\dimen@-\@tempdima
    \ifdim\dimen@>\z@ \dimen@0pt\fi
    \rel@kern{0.6}\kern-\dimen@
    \if#31
      \overline{\rel@kern{-0.6}\kern\dimen@\macc@nucleus\rel@kern{0.4}\kern\dimen@}%
      \advance\dimen@0.4\dimexpr\macc@kerna
      \let\final@kern#2%
      \ifdim\dimen@<\z@ \let\final@kern1\fi
      \if\final@kern1 \kern-\dimen@\fi
    \else
      \overline{\rel@kern{-0.6}\kern\dimen@#1}%
    \fi
  }%
  \macc@depth\@ne
  \let\math@bgroup\@empty \let\math@egroup\macc@set@skewchar
  \mathsurround\z@ \frozen@everymath{\mathgroup\macc@group\relax}%
  \macc@set@skewchar\relax
  \let\mathaccentV\macc@nested@a
  \if#31
    \macc@nested@a\relax111{#1}%
  \else
    \def\gobble@till@marker##1\endmarker{}%
    \futurelet\first@char\gobble@till@marker#1\endmarker
    \ifcat\noexpand\first@char A\else
      \def\first@char{}%
    \fi
    \macc@nested@a\relax111{\first@char}%
  \fi
  \endgroup
}
\makeatother

%% file: _I.Introduction/intro.tex
Trajectory optimization is a cornerstone of robotics, enabling autonomous systems to generate goal-oriented motions consistent with their dynamics. Yet traditional nonlinear programming often struggles with the challenges inherent in real-world robotics tasks, such as non-convex objectives, complex nonlinear dynamics, and high-dimensional state-control spaces. Motivated by these challenges, diffusion-based planners have emerged as a compelling paradigm for trajectory optimization, viewing planning as probabilistic inference over trajectories and generating low-cost solutions by progressively denoising samples~\cite{janner_planning_2022, carvalho_motion_2023}.

Model-Based Diffusion~(MBD)~\cite{pan_model-based_2024} strengthens this paradigm by replacing learned score networks with principled scores derived from the system dynamics and cost function. This \emph{training-free} approach avoids collecting large datasets of expert demonstrations and naturally supports \emph{test-time} adaptation to new tasks. However, applying MBD to constrained kinodynamic planning faces two fundamental issues: (i) \emph{sampling inefficiency}, since feasibility and safety constraints concentrate probability mass onto a thin manifold, and (ii) the \emph{absence of safety guarantees}, which is unacceptable in safety-critical applications. 

\begin{figure*}[t]
\centering
\includegraphics[width=0.95\linewidth]{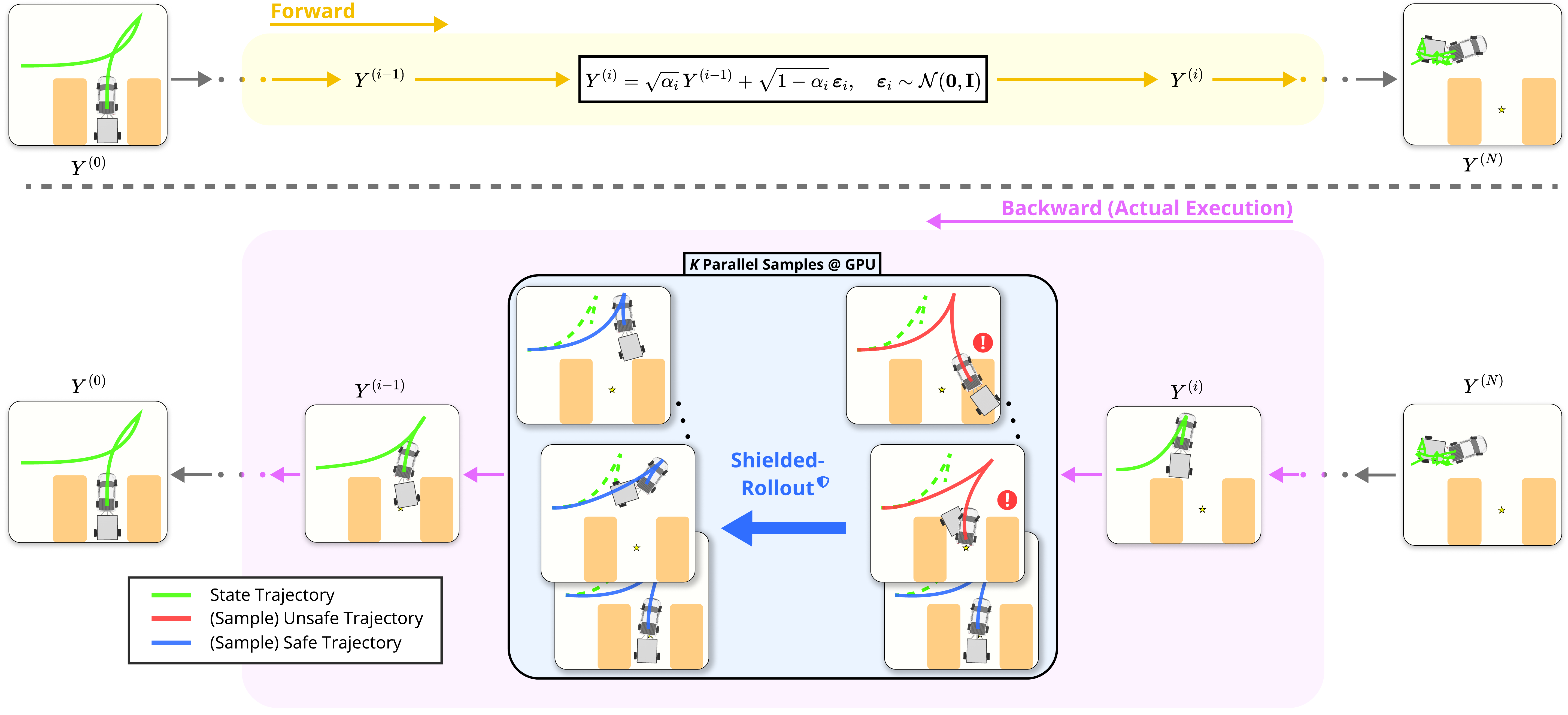}
    \caption{Overview of the Safe Model Predictive Diffusion (Safe MPD$^{\shield}$) algorithm. (a) The forward process gradually adds noise to an optimal trajectory. (b) Reverse (denoising) process with shielded rollout: from the current noisy estimate $Y^{(i)}$, $K$ perturbed candidates are drawn; some are initially unsafe (e.g., collisions or jackknifing for the tractor-trailer). Shielded rollout transforms each candidate into a kinodynamically feasible and safe trajectory, after which weighted averaging and score ascent update $Y^{(i-1)}$.
    }
    \label{fig:main}
\vspace{-5pt}
\end{figure*}

\subsection{Related Work}\label{sec:related}
There have been several attempts to impose safety on trajectories generated by diffusion models, primarily in model-free contexts. A straightforward approach is \textbf{filtering}, where generated trajectories are discarded if they are deemed unsafe~\cite{yun_guided_2024, qing_bitrajdiff_2025}. This is generally sample-inefficient because it may reject a large number of samples, and it is especially problematic when feasible trajectories occupy only a thin manifold. Another popular technique is \textbf{guidance}, which steers the denoised trajectory towards safe regions by gradient descent~\cite{janner_planning_2022, zhong_guided_2023, lee_learning_2025, huang_diffuse-cloc_2025}. While effective, this \emph{post-processing correction} can result in kinodynamically infeasible trajectories, undermining a critical requirement for many robotic systems to reliably execute planned motions. Furthermore, constructing a differentiable landscape for the safety objective is often non-trivial, particularly with non-convex obstacles or complex robot geometries. A third category of methods performs hard \textbf{projection} of the states onto the safe set~\cite{christopher_constrained_2024, christopher_neuro-symbolic_2025, romer_diffusion_2025}. This process can be computationally intensive, especially when the projection operation is non-convex and the state space is high-dimensional. Finally, some methods integrate \textbf{barrier functions} directly into the diffusion process~\cite{ma_constraint-aware_2025,xiao_safediffuser_2025}. For instance, SafeDiffuser proposes solving a Quadratic Program~(QP) with Control Barrier Function~(CBF) constraints at each denoising step~\cite{xiao_safediffuser_2025}. The computational overhead introduced by this method renders it intractable within the model-based diffusion framework, as it requires solving QPs for every parallel trajectory sample.

\subsection{Contributions}

In this paper, we propose a novel diffusion-based trajectory optimization algorithm, Safe Model Predictive Diffusion (Safe MPD$^{\shield}$). The main contributions of this work are:
\begin{itemize}
\item We propose Safe MPD$^{\shield}$ that integrates a safety shield directly into the diffusion process to guarantee kinodynamic feasibility and safety by construction.

\item We empirically show a dramatic improvement in sample efficiency by enforcing that all sampled trajectories within the denoising process are both feasible and safe.

\item We demonstrate that our approach is highly computationally efficient, achieving \emph{sub-second} planning times through a parallelized GPU implementation of our shielding mechanism.

\item We validate our method on challenging non-convex planning tasks including tractor-trailer systems, showing that it outperforms existing safety strategies, without requiring model-specific hyperparameter tuning.
\end{itemize}

%% file: _II.Preliminaries/a_problem.tex

Consider a discrete-time nonlinear system:
\begin{equation}
\vx_{t+1} = f(\vx_{t}, \vu_{t}),
\label{eq:dt_dynamics}
\end{equation}
where $\vx_{t} \in \StateSpace \subset \Rn$ is the state at time step $t \in \Natural^{+}$, and $\vu_{t} \in \ControlSpace \subset \Rm$ is the control input at time step $t$, with $\ControlSpace$ being a set of admissible controls for System~\eqref{eq:dt_dynamics}. The function $f: \StateSpace \times \ControlSpace \to \StateSpace$ represents the dynamics. 

Given the initial condition $\vx_{0} = \vx_{\textup{init}}$, we aim to solve a general trajectory optimization problem using diffusion:
\begin{subequations} \label{eq:to}
\begin{align}
\{\vx^{*}_{t}\}_{t=1}^{\horizon}, \{\vu^{*}_{t}\}_{t=0}^{\horizon-1} \coloneqq & \argmin_{\vx_{1:\horizon}, \vu_{0:\horizon-1}} \, J(Y) \label{eq:to_objective}  \\
\text{s.t.} \quad \vx_{t+1} = f(\vx_{t}, \vu_{t}), \ & t = 0,...,\horizon-1 ,\label{eq:to_dt_dynamics} \\
g(\vx_{t}) \leq 0, \ & t = 0,...,\horizon ,\label{eq:to_inequality_constraint} \\
\vu_{t} \in \ControlSpace, \ & t = 0,...,\horizon-1 . \label{eq:to_input_constraint}
\end{align}
\end{subequations}
We denote $\tau_{\vx} \coloneqq [\vx_{1},\dots,\vx_{\horizon}]$ as the state trajectory and $\tau_{\vu} \coloneqq [\vu_{0},\dots,\vu_{\horizon-1}]$ as the control sequence. We use $Y \coloneqq \{\tau_{\vx}, \tau_{\vu} \}$ to denote all decision variables. The objective function $J$ in \eqref{eq:to_objective} is defined as $J(Y) = \terminalcost(\vx_{\horizon}) + \sum_{t=0}^{\horizon-1} \runningcost(\vx_{t}, \vu_{t})$, where $\runningcost: \StateSpace \times \ControlSpace \to \RealSpace $ and $\terminalcost: \StateSpace \to \RealSpace$ are the stage cost and the terminal cost. $g: \StateSpace \to \RealSpace$ represents the inequality constraint. The parameter $\horizon$ represents the planning horizon.

Classical approaches solve \eqref{eq:to} via nonlinear programming, but these methods often struggle with non-convex objectives and constraints, complex nonlinear dynamics, and high-dimensional state space~$\StateSpace$ and control space~$\ControlSpace$. 
Recently, an alternative paradigm that bypasses these difficulties has gained significant interest by recasting trajectory optimization as a sampling problem and directly generating samples from the optimal trajectory distribution~\cite{janner_planning_2022, carvalho_motion_2023, pan_model-based_2024, lee_learning_2025}.



%% file: _II.Preliminaries/b_diffusion.tex
\textbf{Trajectory Optimization as Probabilistic Sampling:} We associate to \eqref{eq:to} a target distribution~$p_{0}$ on trajectories that factors into (i) optimality~\eqref{eq:to_objective}, (ii) dynamical feasibility~\eqref{eq:to_dt_dynamics}, and (iii) constraint satisfaction~\eqref{eq:to_inequality_constraint}:
\begin{equation}
p_{0}(Y) \propto p_{J}(Y) \, p_{f}(Y) \, p_{g}(Y) ,
\label{eq:target_distribution}
\end{equation}
with 
\begin{subequations}
\begin{align}
p_{J}(Y) &\propto \exp \left(- J(Y)/\lambda \right), \, \lambda>0, \label{eq:optimality_distribution} \\
p_{f}(Y) &\propto \prod_{t=0}^{\horizon-1} \delta \left(\vx_{t+1} = f(\vx_{t},\vu_{t}) \right), \label{eq:feasibility_distribution}\\
p_{g}(Y) &\propto \prod_{t=0}^{\horizon} \IndicatorFunction \left( g(\vx_{t}) \le 0 \right),\label{eq:constraint_distribution}
\end{align}
\end{subequations}
where $\delta(\cdot)$ is a Dirac delta function and $\IndicatorFunction(\cdot)$ is an indicator function. In words, $p_{0}$ places probability mass on trajectories that are kinodynamically feasible and satisfying constraints, and exponentially favors low-cost ones. Obtaining the solution $Y^{*}$ from solving the trajectory optimization problem in \eqref{eq:to} is equivalent to sampling from the target distribution in \eqref{eq:target_distribution} given a low temperature $\lambda \to 0$~\cite{pan_model-based_2024}.

In trajectory optimization tasks, since the dynamics~$f$ and objective function~$J$ are known, we can evaluate the unnormalized probability $p_{0}(Y)$ for any given trajectory $Y$. However, sampling directly from the target distribution $p_{0}$ is generally intractable because of its high dimensionality and the sparsity of the feasible manifold.

%% file: _II.Preliminaries/c_mbd.tex
To effectively sample from the target distribution $p_{0}$, diffusion iteratively refines samples starting from pure noise, which can be easily drawn from an isotropic Gaussian distribution. These denoising steps are referred to as the backward (or reverse) process, which reverses a predefined forward process that gradually corrupts data into pure noise~\cite{ho_denoising_2020}.

\textbf{Forward (noising) process:} Let us denote the variance schedule as $\{\beta_{i} \in (0,1)\}_{i=1}^{N}$, then $\alpha_{i} \coloneqq 1-\beta_{i}$, and $\bar{\alpha}_{i} \coloneqq \prod_{k=1}^{i}\alpha_{k}$~\cite{ho_denoising_2020}. Starting from $Y^{(0)} \sim p_{0}(\cdot)$, the forward Markov chain is
\begin{equation}
Y^{(i)} = \sqrt{\alpha_{i}} \, Y^{(i-1)} + \sqrt{1-\alpha_{i}} \, \vepsilon_{i},
\quad \vepsilon_{i}\sim\calN(\mathbf{0},\mathbf{I}),
\label{eq:forward_process}
\end{equation}
so that
\begin{equation}
p_{i\mid i-1}(\cdot \mid Y^{(i-1)}) \sim \calN \left(\sqrt{\alpha_{i}} \, Y^{(i-1)},(1-\alpha_{i}) \mathbf{I} \right).
\label{eq:forward_one_step}
\end{equation}
Since the noise at each noising step is independent, it yields
\begin{equation}
p_{i\mid 0}(\cdot \mid Y^{(0)}) \sim \calN \left(\sqrt{\bar{\alpha}_{i}} \, Y^{(0)},(1-\bar{\alpha}_{i}) \mathbf{I} \right).
\label{eq:forward_n_step}
\end{equation}


\textbf{Reverse (denoising) process via Monte Carlo score ascent:} The reverse process aims to recover a sample to the target distribution $p_{0}$ by starting with a sample drawn from $Y^{(N)} \sim \calN(\mathbf{0},\mathbf{I})$. This involves iteratively evaluating the posterior distribution, where the marginal is given by:
$p_{i-1}(Y^{(i-1)}) = \int p_{i-1\mid i}(Y^{(i-1)}\mid Y^{(i)})p_{i}(Y^{(i)}) \, \operatorname{d}Y^{(i)}.$
Unlike standard (model-free) diffusion models that rely on a neural network learned with a large number of data to estimate the score function~\cite{ho_denoising_2020}, Model-Based Diffusion~(MBD) exploits the prior information of the dynamics model~$f$ and objective function~$J$ to evaluate the score directly~\cite{pan_model-based_2024}. 

In each denoising step from $i$ to $i-1$, MBD performs one-step gradient ascent on $\log p_{i}(Y^{(i)})$:
\begin{equation}
Y^{(i-1)} = \frac{1}{\sqrt{\alpha_{i}}} \left( Y^{(i)} + (1 - \bar{\alpha}_{i}) \nabla_{Y^{(i)}} \log p_{i}(Y^{(i)}) \right) .
\label{eq:reverse_one_step}
\end{equation}
The crucial insight of MBD is to estimate the score function $\nabla_{Y^{(i)}} \log p_{i}(Y^{(i)})$ using a Monte Carlo approximation:
\begin{equation}
\nabla_{Y^{(i)}} \log p_{i}(Y^{(i)}) \approx -\frac{Y^{(i)}}{1-\bar{\alpha}_{i}} + \frac{\sqrt{\bar{\alpha}_{i}}}{1-\bar{\alpha}_{i}} \widebar{Y}^{(i)}.
\label{eq:score_approx}
\end{equation}
Here, $\widebar{Y}^{(i)}$ is a weighted average of a batch of candidate samples ${\calY^{(i)}_{k}}$, $k=1, \ldots,K$, which are drawn from the proposal distribution
\begin{equation} \label{eq:diffusion_distribution}
\calY^{(i)}_{k} \sim \calN \left(\frac{Y^{(i)}}{\sqrt{\bar{\alpha}_{i}}}, \, \left(\frac{1}{\bar{\alpha}_{i}} - 1\right)\mathbf{I} \right), 
\end{equation}
and then averaged with the weights evaluated using the known target distribution $p_{0}$:
\begin{equation} \label{eq:weighted_sum}
\widebar{Y}^{(i)} \coloneqq \frac{\sum_{k=1}^{K} \, \calY^{(i)}_{k} \, p_{0}(\calY^{(i)}_{k})}{\sum_{k=1}^{K} \, p_{0}(\calY^{(i)}_{k})}
\end{equation}

In essence, at each denoising step, MBD generates a set of potential trajectories that are perturbed around the current noisy estimate, scores them using the true objective and constraints encoded in $p_{0}$, and uses this information to perform a principled update, guiding the sample from noise toward an improved solution. With the number of diffusion steps $N=1$, MBD reduces to the Cross Entropy Method (CEM)~\cite{rubinstein_cross-entropy_1999}.

\begin{remark}
As the generation and evaluation of these $K$ candidate samples are independent, this process is highly parallelizable. With the aid of modern GPUs/TPUs, each denoising step can be significantly accelerated~\cite{williams_model_2017}.
\end{remark}

\begin{remark}
MBD can be used for generic optimization problems other than trajectory optimization. 
\end{remark}


%% file: _III.Methodology/_intro.tex
Although MBD can map noise to low-cost trajectories without learned neural networks, its direct application to \eqref{eq:to} reveals two fundamental challenges.

First, sampling efficiency can be extremely low. Recall the target distribution is given as $p_{0}(Y) \propto p_{J}(Y) \, p_{f}(Y) \, p_{g}(Y)$. The Dirac delta function for equality constraint~$p_{f}(\cdot)$ and the indicator function for inequality constraint~$p_{g}(\cdot)$ assign non-zero probability density only to kinodynamically feasible and constraint satisfying trajectories, which form a thin manifold of Lebesgue measure zero. Consequently, nearly all samples~${\calY^{(i)}_{k}}$ will receive zero weights, rendering the Monte Carlo update ineffective.

Second, while MBD inherits the computational advantages of sampling-based methods like CEM~\cite{rubinstein_cross-entropy_1999} and MPPI~\cite{williams_model_2017}, it also shares their fundamental drawback: a lack of safety guarantees. Even if the candidate samples are filtered to be safe, the weighted averaging step \eqref{eq:weighted_sum} does not guarantee that the resulting updated trajectory will remain safe~\cite{yin_shield_2023}. Also, there is no straightforward way in these frameworks to ensure that the system can be rendered safe for all future time from the terminal state~$\vx_{\horizon}$.

In this work, we introduce \textbf{\emph{Safe Model Predictive Diffusion (Safe MPD$^{\shield}$)}} to address these limitations: (i) we ensure that all trajectory samples at every denoising step are feasible~(\autoref{subsec:mpd}) and safe~(\autoref{subsec:shielding}-\autoref{subsec:main}), dramatically improving sample efficiency; and (ii) we provide formal guarantees that the final diffusion trajectory~$Y^{(0)}$ is both kinodynamically feasible and safe (\autoref{subsec:main}).

%% file: _III.Methodology/a_mpd.tex
We first show how to ensure all trajectory samples are \textbf{\emph{kinodynamically feasible}}. For each noisy candidate $\calY_{k}$ drawn from \eqref{eq:diffusion_distribution}, we extract its control sequence $[\vu_{0},\dots,\vu_{\horizon-1}]$. This control sequence is then simulated forward from the initial state $\vx_{0}$ using the dynamics model $f$, producing a new kinodynamically feasible trajectory $\calY^{f}_{k}$. This feasible trajectory then replaces the original sample $\calY_{k}$ in the weighted averaging step~\eqref{eq:weighted_sum}. This technique is analogous to shooting methods~\cite{betts_practical_2010} and MPPI~\cite{williams_model_2017}, and was first adapted for MBD in \cite{pan_model-based_2024}. We refer to this framework as \textbf{\emph{Model Predictive Diffusion (MPD)}}. However, while MPD guarantees kinodynamic feasibility, it still relies on evaluating samples using $p_{J}(\cdot) \, p_{g}(\cdot)$, which remains sample-inefficient and lacks formal safety guarantees.

%% file: _III.Methodology/b_shielding.tex
Now, we show how to ensure all trajectory samples and the final trajectory from MPD are safe. Unlike prior methods described in \autoref{sec:related} that rely on post-processing corrections and can be computationally expensive, we propose \textbf{\emph{Shielded Rollout}} to ensure safety, which can be seamlessly integrated into the MPD process. Inspired by model predictive shielding~\cite{bastani_safe_2021-1} and gatekeeper \cite{agrawal_gatekeeper_2024}, the proposed shielded rollout ensures the system remains \textbf{\emph{safe for all future time}}. 


To define the shielded rollout process, we first define the backup policy~$\pi_{\textup{backup}}$ and the corresponding sets. We denote $\calS$ as the set of \emph{(instantaneously) safe} states, i.e.,
\begin{equation}\label{eq:safe_set}
\calS \coloneqq \{\vx \in \StateSpace \mid g(\vx) \leq 0\}.
\end{equation}

\begin{definition}[Controlled-Invariant Set]\label{def:controlled_invraiant_set}
Given a policy $\pi:\StateSpace \to \ControlSpace$, a set
$\calC \subseteq \calS$ is \emph{controlled-invariant} under $\pi$ if for all $\vx_{t_{0}} \in \calC$, the solution $\vx_{t}$ of the closed-loop system $\vx_{t+1} = f(\vx_{t},\pi(\vx_{t}))$ from initial condition $\vx_{0} = \vx_{t_{0}}$ satisfies $\vx_
{t} \in \calC$ for all $t \geq t_{0}$.
\end{definition}

\begin{definition}[Invariance Policy]\label{def:invariance_policy}
An \emph{invariance policy} for $\calC$ is a policy $\pi_{\textup{inv}}: \calC \to \ControlSpace$ that renders $\calC$ controlled-invariant.
\end{definition}

\begin{definition}[Recovery Policy]\label{def:recovery_policy}
A policy $\pi_{\textup{rec}}:\calS \to \ControlSpace$ is called a \emph{recovery policy} to a set $\calC \subseteq \calS$ if, for the closed-loop system $\vx_{t+1}=f(\vx_{t},\pi_{\textup{rec}}(\vx_{t}))$, the set $\calC$ is reachable from any state in $\calS$ within a fixed time $T_{B}<\infty$, i.e.,
\begin{equation}
\vx_{t_{0}} \in \calS \implies \vx_{t_{0}+T_{B}} \in \calC .
\end{equation}
\end{definition}

Given an invariance policy $\pi_{\textup{inv}}$ and a recovery policy $\pi_{\textup{rec}}$ for $\calC$, we can define the backup policy~$\pi_{\textup{backup}}: \calS \to \ControlSpace$ as
\begin{equation}\label{eq:backup_policy}
\pi_{\textup{backup}}(\vx)=
\begin{cases}
\pi_{\textup{inv}}(\vx), & \text{if} \, \vx \in \calC\\
\pi_{\textup{rec}}(\vx), & \text{otherwise}.
\end{cases}
\end{equation}
We assume that a backup policy $\pi_{\textup{backup}}$ for $\calC$ is known, which can often be designed using established methods such as simplex architectures or reachability analysis~\cite{wabersich_data-driven_2023}. Next, we define the condition of a \emph{valid} state trajectory, which enables the construction of a computationally efficient monitor.
\begin{definition}[Valid]\label{def:valid}
A state trajectory~$\hat{\tau}_{\vx}=[\vx_{0},\dots,\vx_{T_{B}}]$ is \emph{valid} if the trajectory is safe w.r.t. the safe set~$\calS$ over a finite interval:
\begin{equation}
\vx_{t} \in \calS\ \quad \forall t \in \{0, \dots, T_{B}\},
\end{equation}
and the trajectory reaches $\calC$ at $T_{B}$:
\begin{equation}
\vx_{T_{B}} \in \calC .
\end{equation}
\end{definition}

\begin{remark}
Critically, checking whether a trajectory is valid only requires numerical forward integration of the closed-loop system over the finite interval $\{0,\ldots,T_{B}\}$. This is computationally lightweight and naturally parallelizable across trajectory samples.
\end{remark}

\begin{algorithm}[t]\footnotesize
\caption{\texttt{Rollout}($\vx_{0}, \pi, T$)}
\label{alg:rollout}
\DontPrintSemicolon
\KwIn{Initial state~$\vx_0$; policy~$\pi$; rollout horizon~$T$}
\KwOut{State trajectory $\tau_{\vx}$}
\BlankLine
$\vx \gets \vx_{0}$; \, $\tau_{\vx}[0] \gets \vx_{0}$\;
\For{$t \gets 0$ \KwTo $T-1$}{%
    $\vx \gets f(\vx,\pi(\vx))$\;
    $\tau_{\vx}[t+1] \gets \vx$\;
}%
\Return{$\tau_{\vx}$}
\end{algorithm}

\begin{algorithm}[t]\footnotesize
\caption{\texttt{Shielded-Rollout$^{\shield}$($\vx_{0}, \tau_{\vu}$)}}
\label{alg:shielded_rollout}
\DontPrintSemicolon
\KwIn{Initial state~$\vx_0$; nominal control sequence~$\tau_{\vu}$; backup policy~$\pi_{\textup{backup}}$}
\KwOut{Shielded state trajectory $\tau_{\vx}^{\shield}$; Shielded control sequence $\tau_{\vu}^{\shield}$}
\BlankLine
$\vx \gets \vx_0$; \, $\tau_{\vx}^{\shield}[0] \gets \vx_{0}$\;
\For{$t \gets 0$ \KwTo $\horizon-1$}{%
    $\vu_{\textup{nom}} \gets \tau_{\vu}[t]$\;
    $\hat{\vx} \gets f(\vx,\vu_{\textup{nom}})$\;
    $\hat{\tau}_{\vx} \gets \texttt{Rollout}(\hat{\vx}, \pi_{\textup{backup}}, T_{B})$\;
    \If()%
       {$\hat{\tau}_{\vx}$ is valid by \autoref{def:valid}}{%
            $\vx \gets \hat{\vx}$; \, $\tau_{\vx}^{\shield}[t+1] \gets \hat{\vx}$; \, $\tau_{\vu}^{\shield}[t] \gets \vu_{\textup{nom}}$\;
    }\Else{%
        \For(\tcp*[f]{fallback to backup})%
            {$t' \gets t$ \KwTo $\horizon-1$}{%
            $\vu_{\textup{backup}} \gets \pi_{\textup{backup}}(\vx)$; \, $\tau_{\vu}^{\shield}[t'] \gets \vu_{\textup{backup}}$\;
            $\vx \gets f(\vx,\vu_{\textup{backup}})$; \, $\tau_{\vx}^{\shield}[t'+1] \gets \vx$\;
        }
        \textbf{break}\;
    }%
}%
\Return{$\tau_{\vx}^{\shield}$, $\tau_{\vu}^{\shield}$}
\end{algorithm}

The proposed shielded rollout takes in a potentially unsafe nominal control sequence $\tau_{\vu}$ and produces a provably-safe trajectory~$\tau_{\vx}^{\shield}$. Although the system starts within $\calC$, the nominal control input (e.g., from diffusion), which is not drawn from $\pi_{\textup{inv}}$, may attempt to drive the system outside the safe set $\calS$ to achieve task objectives. To prevent this, our method acts as a safety shield. At each time step $t$, it first computes the prospective next state~$\hat{\vx}_{t+1} \coloneqq f(\vx_{t},\vu_{\textup{nom},t})$ for the nominal input $\vu_{\textup{nom}, t} \in \tau_{\vu}$. It then performs a $T_{B}$-step rollout from $\hat{\vx}_{t+1}$ using the backup policy~$\pi_{\textup{backup}}$ and checks the validity of the simulated trajectory~$\hat{\tau}_{\vx}$ as in \autoref{def:valid}. The standard rollout is shown in \autoref{alg:rollout}. If valid, the nominal input~$\vu_{\textup{nom}, t}$ is accepted; otherwise, the system switches to $\pi_{\textup{backup}}$ for the remainder of the horizon. The procedure is outlined in detail in \autoref{alg:shielded_rollout}.

\begin{assumption}
Our analysis is based on a discrete-time formulation. We assume that the system's continuous trajectory between two consecutive safe states, $\vx_{t}$ and $\vx_{t+1}$, remains within the safe set~$\calS$.
\end{assumption}

\begin{theorem}[\texttt{Shielded-Rollout$^{\shield}$}]\label{thm:shield}
Given any initial state~$\vx_{0} \in \calC$, the shielded state trajectory~$\tau_{\vx}^{\shield}$ generated by \autoref{alg:shielded_rollout} enables the system to remain in the safe set~$\calS$ for all future time, i.e., $t\geq0$.
\end{theorem}

\begin{proof}
We prove the claim by induction.

\textit{Base Case ($t=0$)}: By assumption, $\vx_{0} \in \calC \subseteq \calS$.

\textit{Induction Step:} Assume that $\vx_{t}\in\calS$ for some $t\in\{1,\ldots,\horizon-1\}$. We now show that the subsequent state~$\vx_{t+1}$ also remains in $\calS$. As described in \autoref{alg:shielded_rollout}, the algorithm evaluates the nominal control input~$\vu_{\textup{nom}, t} \in \tau_{\vu}$ by first computing $\hat{\vx}_{t+1}=f(\vx_{t},\vu_{\textup{nom}, t})$. This leads to two cases:

\begin{enumerate}
\item \textit{Case 1 (Valid nominal step):}  
If the simulated trajectory $\hat{\tau}_{\vx}$ from $\texttt{Rollout}(\hat{\vx}_{t+1}, \pi_{\textup{backup}}, T_{B})$ is valid, \autoref{def:valid} guarantees $\hat{\vx}_{t+1} \in \calS$ \emph{and} $\pi_{\textup{backup}}$ can drive $\hat{\vx}_{t+1}$ into $\calC$ within $T_{B}$ while staying in $\calS$. Hence $\vx_{t+1}:=\hat{\vx}_{t+1} \in \calS$ and $\calC$ is reachable no later than $T_{B}$.

\item \textit{Case 2 (Invalid nominal step):}  
If it was not valid, \texttt{Shielded-Rollout$^{\shield}$} applies $\pi_{\textup{backup}}$. Since the nominal control inputs~$\vu_{\textup{nom}, (\cdot)}$ in previous steps would have been applied only when the simulated trajectories were valid, we have $\vx_{t+1} \in \calS$ by \autoref{def:valid}. Moreover, the system reaches $\calC$ within $T_{B}$ under $\pi_{\textup{rec}}$.
\end{enumerate}

Thus $\vx_{t+1} \in \calS$ in all cases, completing the induction. 

Moreover, the system at $\vx_{\horizon}$ can enter $\calC$ after at most time $T_{B}$. Then $\pi_{\textup{inv}}$ ensures $\vx_{t} \in \calC \subseteq \calS$ for all future time $t \geq \horizon + T_{B}$.
\end{proof}

%% file: _III.Methodology/c_safe.tex
\begin{algorithm}[t]\footnotesize
\caption{Safe Model Predictive Diffusion$^{\shield}$}
\DontPrintSemicolon
\KwIn{noise schedule $\{\bar{\alpha}_i\}_{i=1}^{N}$; denoising steps $N$; number of samples $K$}
\KwOut{safe optimized trajectory}

\BlankLine
\textbf{Initialization:} draw $Y^{(N)}\sim\mathcal{N}(\mathbf{0},\mathbf{I})$

\For(){$i \leftarrow N\ \textbf{to}\ 1$}{
    \tcp{-- Denoising step -- //}
    
    \textbf{Sampling:} $\calY^{(i)}_{1:K} \, \sim \, \mathcal{N}\left(\frac{Y^{(i)}}{\sqrt{\bar{\alpha}_{i}}},\,
    \left(\frac{1}{\bar{\alpha}_{i}}-1\right)\mathbf{I}\right)$ \\[3pt]
    
    \textbf{Shielded Rollout:} $\calY^{(i) \shield}_{1:K}\leftarrow\texttt{Shielded-Rollout}^{\shield} (\calY^{(i)}_{1:K})$ \\[3pt]


    \textbf{Weighted sum:} $\widebar{Y}^{(i)} \coloneqq \frac{\sum_{k=1}^{K} \, \calY^{(i) \shield}_{k} \, p_{0}(\calY^{(i) \shield}_{k})}{\sum_{k=1}^{K} \, p_{0}(\calY^{(i) \shield}_{k})}$ \\[3pt]

        \textbf{Score estimate:} 
    $\nabla_{Y^{(i)}} \log p_{i}(Y^{(i)}) \approx -\frac{Y^{(i)}}{1-\bar{\alpha}_{i}} + \frac{\sqrt{\bar{\alpha}_{i}}}{1-\bar{\alpha}_{i}} \widebar{Y}^{(i)}$ \\[3pt]
    
    \textbf{Score-based update:} \\[3pt]
    \qquad $Y^{(i-1)}\leftarrow
     \frac{1}{\sqrt{\alpha_{i}}} \left( Y^{(i)} + (1 - \bar{\alpha}_{i}) \nabla_{Y^{(i)}} \log p_{i}(Y^{(i)}) \right) $
}
\Return{\texttt{\textup{Shielded-Rollout}}$^{\shield}(Y^{(0)})$}
\label{alg:safe_mpd}
\end{algorithm}

Finally, we present our main algorithm, Safe MPD$^{\shield}$. A schematic overview is shown in Fig.~\ref{fig:main}, and as detailed in \autoref{alg:safe_mpd}, it integrates \texttt{Shielded-Rollout$^{\shield}$} at two critical stages.

\textbf{Within the diffusion process: }
Within each denoising step, all $K$ candidate trajectory samples $\calY^{(i)}_{1:K}$ are passed through \texttt{Shielded-Rollout$^{\shield}$} to produce $\calY^{(i) \shield}_{1:K}$ that lie strictly on the feasible and safe manifold. This offers two significant advantages for trajectory optimization. (i) First, since every shielded sample $\calY^{(i) \shield}_{k}$ is guaranteed to be safe and feasible, the probability terms for feasibility and safety are constant across all such samples and can thus be disregarded from the original target distribution \eqref{eq:target_distribution}. Specifically, for any two samples, $k_{1}, k_{2} \in \{1,\ldots,K\}$: 
\begin{equation}
p_{f}(\calY^{(i) \shield}_{k_{1}}) = p_{f}(\calY^{(i) \shield}_{k_{2}}) \,\, \text{and} \,\, p_{g}(\calY^{(i) \shield}_{k_{1}}) = p_{g}(\calY^{(i) \shield}_{k_{2}}).
\end{equation}
The target distribution in Safe MPD$^{\shield}$ therefore simplifies to depend solely on the optimality factor:
\begin{equation} \label{eq:new_target_distribution}
p_{0}(Y) \propto p_{J}(Y) .
\end{equation}
This \emph{improves sample efficiency}, since no computational effort is wasted on samples that would otherwise receive zero weight due to constraint violations. (ii) Second, the diffusion process complements the conservative nature of safety filters~\cite{agrawal_gatekeeper_2024, romer_diffusion_2025}. Applying safety filters only at the last layer of the trajectory generation often overly constrains the solution, reducing overall performance. The iterative score ascent continuously pushes the trajectory distribution toward lower-cost regions, which we find empirically helps the optimizer overcome local minima while shielding preserves safety.

\textbf{On the final trajectory: }
Furthermore, we apply \texttt{Shielded-Rollout$^{\shield}$} to the final trajectory $Y^{(0)}$ from the diffusion process. This guarantees that the trajectory returned by Safe MPD$^{\shield}$ is kinodynamically feasible and safe, satisfying all inequality constraints by construction during its execution. Moreover, it guarantees that the system can be rendered safe from the terminal state $\vx_{\horizon}$ for all future time.

%% file: _V.Experiments/_intro.tex
Our experimental evaluations aim to answer the following key questions: \textbf{(Q1)} Can our method solve complex, non-convex trajectory optimization problems with kinodynamic constraints, where safety depends on factors like inertia and acceleration limits? \textbf{(Q2)} Is our method scalable to different dynamical systems without requiring model-specific hyperparameter tuning? \textbf{(Q3)} Can the proposed shielded rollout be integrated into the MPD framework in a computationally efficient manner? \textbf{(Q4)} Are the resulting trajectories kinodynamically feasible and executable by tracking controllers?

%% file: _V.Experiments/a_dynamics.tex
To address \textbf{Q1} and \textbf{Q2}, we evaluate our algorithm on a series of increasingly challenging dynamical models. The problems we consider involve scenarios where safety cannot be na\"ively achieved by simply stopping due to the dynamics of high-order systems and \emph{input constraints}. The evaluation task is an automated parking scenario, where the vehicle must navigate a cluttered environment with $N_{\textup{obs}}$ obstacles modeled as rectangles and circles to reach a target configuration. We denote the set of obstacles as $\calO \coloneqq \bigcup_{j=1}^{N_{\textup{obs}}} \calO_{j}$.

We consider: (i) a kinematic bicycle, (ii) a kinematic tractor-trailer system, and (iii) an acceleration-controlled tractor-trailer system. The dynamics of the discrete-time kinematic tractor-trailer model with sampling time $T_{s} > 0$ is given by:
\begin{align}
p^{x}_{t+1} &= p^{x}_{t} + T_{s} v_{t} \cos(\theta^{1}_{t}), \, p^{y}_{t+1} = p^{y}_{t} + T_{s} v_{t} \sin(\theta^{1}_{t}), \label{eq:kinematic_bicycle_position} \\
\theta^{1}_{t+1} &= \theta^{1}_{t} + T_{s} \frac{v_{t}}{\ell_1} \tan(\delta_{t}), \label{eq:kinematic_bicycle_theta1} \\
\theta^{2}_{t+1} &= \theta^{2}_{t} + \nonumber \\
T_{s} & \frac{v_{t}}{\ell_2} \left( \sin(\theta^{1}_{t} - \theta^{2}_{t})- \frac{\ell_h}{\ell_1} \cos(\theta^{1}_{t} - \theta^{2}_{t}) \tan(\delta_{t}) \right). \label{eq:kinematic_tt}
\end{align}
The state of the system is $\vx_{\textup{tt}} = [p^{x}, p^{y}, \theta^{1}, \theta^{2}]^{\top}$, where $(p^{x}, p^{y})$ denotes the position of the rear axle of the tractor, and $\theta^{1}$ and $\theta^{2}$ are the heading angles of the tractor and the trailer, respectively. The control input is $\vu_{\textup{tt}} = [v, \delta]^{\top}$, where $v$ is the longitudinal velocity of the tractor and $\delta$ is the steering angle of the front wheels. The geometric parameters $\ell_{1}, \ell_{2}, \ell_{h}$ denote the tractor wheelbase, trailer length, and hitch length, respectively. The kinematic bicycle model is described simply by \eqref{eq:kinematic_bicycle_position}-\eqref{eq:kinematic_bicycle_theta1}.

We also introduce an acceleration-controlled tractor-trailer system with second-order dynamics (see Fig.~\ref{fig:acc_tt}). We augment the state with velocity and steering angle, yielding $\vx_{\textup{acc-tt}} = [p^{x}, p^{y}, \theta^{1}, \theta^{2}, v, \delta]^{\top}$. The control inputs are now the longitudinal acceleration and the steering rate, $\vu_{\textup{acc-tt}} = [a, \omega]^{\top}$. The dynamics is described by \eqref{eq:kinematic_bicycle_position}-\eqref{eq:kinematic_tt}, and the following:
\begin{equation} \label{eq:acc_tt}
v_{t+1} = v_{t} + T_{s} \, a_{t}, \quad \delta_{t+1} = \delta_{t} + T_{s} \, \omega_{t}
\end{equation}
In all three models, the \emph{control inputs are \textbf{bounded}} ($\ControlSpace \neq \RealSpace^{2}$).

Trajectory planning and control for tractor-trailer systems are challenging, particularly during backward maneuvers, due to their highly nonlinear and unstable dynamics. An additional safety constraint on the hitch angle is also required to prevent jackknifing, i.e., $|\theta^{1}-\theta^{2}| \leq \theta_{\textup{max}}$. Furthermore, the vehicle's body is composed of two disjoint geometries, the tractor $\calR_{\textup{tr}}(\vx) \subset \RealSpace^{2}$ and the trailer $\calR_{\textup{tl}}(\vx) \subset \RealSpace^{2}$ at $\vx$, which makes the collision-free state space non-convex. This non-convexity makes safety enforcement via state projection computationally expensive.

For our shielded rollout algorithm, we design intuitive backup policies. For the kinematic bicycle and tractor-trailer models, where velocity is a control input, both $\pi_{\textup{inv}}$ and $\pi_{\textup{rec}}$ simply apply zero velocity, $v=0$, to stop the system. For the acceleration-controlled tractor-trailer, $\pi_{\textup{rec}}$ applies maximum acceleration $a_{\textup{max}}$ or deceleration $-a_{\textup{max}}$ to drive the velocity $v$ to zero. Once $v=0$ is achieved, $\pi_{\textup{inv}}$ applies $a=0$ to maintain a stationary state. Further examples of backup-controller design for other dynamical systems can be found in~\cite{bastani_safe_2021-1, agrawal_gatekeeper_2024}.

\begin{figure}[tbp]
\centering
\includegraphics[width=0.95\linewidth]{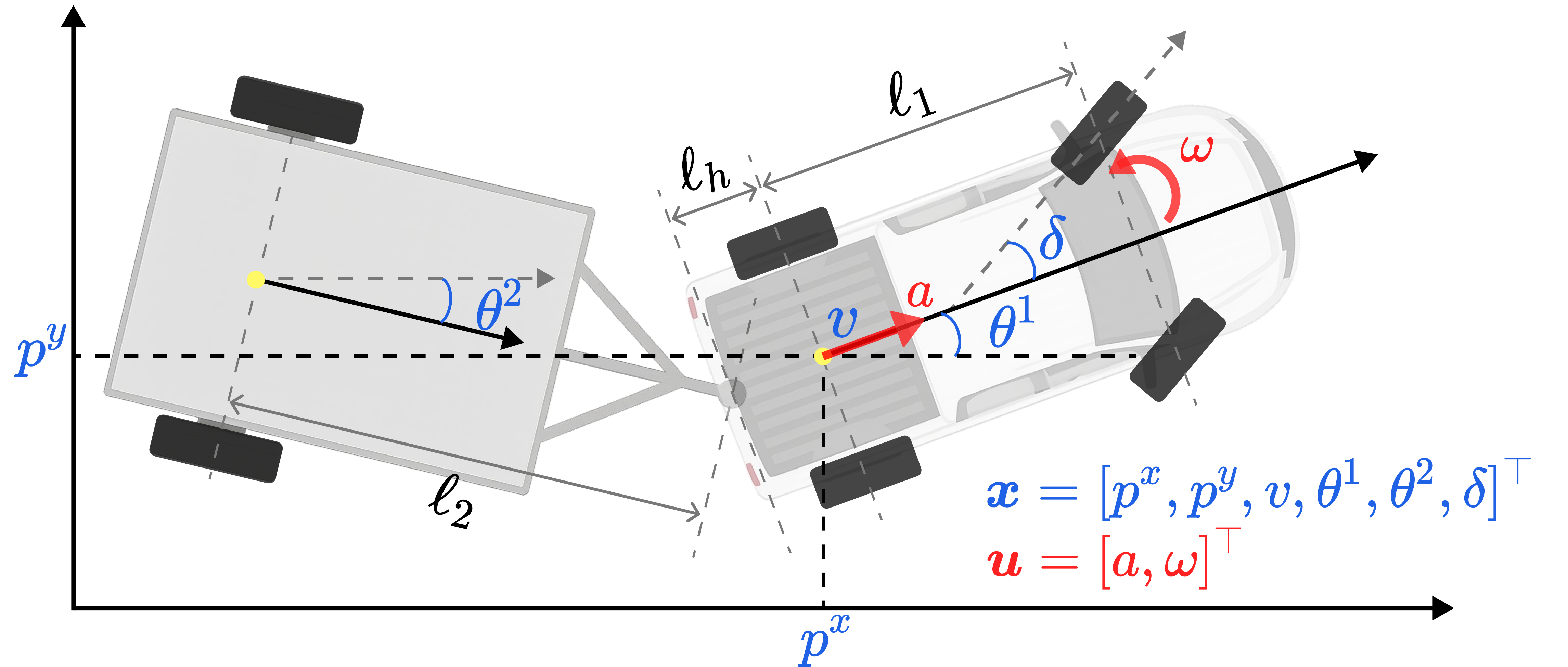}
\caption{
Illustration of an acceleration-controlled tractor-trailer system.}
\label{fig:acc_tt}
\vspace{-5pt}
\end{figure}

%% file: _V.Experiments/b_baseline.tex
\begin{table*}[t]
\centering
\caption{Performance comparison of model-based diffusion planners with different safety strategies across three dynamics models: (A) kinematic bicycle~(Bicycle), (B) kinematic tractor trailer~(TT), and (C) acceleration-controlled tractor trailer~(Accel. TT). For each model, 100 trials were conducted with different initial conditions. The best and comparable results are marked in \textbf{bold}. $^{*}$For the projection method, the environment was simplified to contain only 6 circular obstacles near the goal position, whereas other methods were evaluated with 36 obstacles.}
\begin{tabular}{@{}clccc>{\columncolor{gray!12}[\tabcolsep][1pt]}c@{}}
\toprule
& & \textbf{MPD w/} & \textbf{MPD w/} & \textbf{MPD w/} & \textbf{Safe MPD$^{\shield}$ } \\
\textbf{Metric} & \textbf{Model} & \textbf{Na\"ive Penalty} & \textbf{Projection$^{*}$} & \textbf{Guidance } & \textbf{(Ours)} \\
\midrule
\multirow{3}{*}{\textbf{Success Rate}} 
& (A) Bicycle & \textbf{100\%} & \textbf{100\%} & 89\% & \textbf{100\%} \\
& (B) TT & 64\% & N/A & 51\% & \textbf{100\%} \\
& (C) Accel. TT & 81\% & N/A & 80\% & \textbf{98\%} \\
\cmidrule(l){2-6}
\multirow{3}{*}{\textbf{Safety Violations}} 
& (A) Bicycle & \textbf{0\%} & \textbf{0\%} & 4\% & \textbf{0\%} \\
& (B) TT & 36\% & N/A & 43\% & \textbf{0\%} \\
& (C) Accel. TT & 19\% & N/A & 20\% & \textbf{0\%} \\
\cmidrule(l){2-6}
\multirow{3}{*}{\textbf{Computation Time}} 
& (A) Bicycle & \textbf{0.327 $\pm$ 0.009 s} & 1959.816 $\pm$ 360.710 s & 0.568 $\pm$ 0.016 s & \textbf{0.315 $\pm$ 0.010 s} \\
& (B) TT & \textbf{0.554 $\pm$ 0.027 s} & Time Out & 0.998 $\pm$ 0.029 s & \textbf{0.579 $\pm$ 0.023 s} \\
& (C) Accel. TT & \textbf{0.575 $\pm$ 0.025 s} & Time Out & 0.991 $\pm$ 0.028 s & 1.631 $\pm$ 0.024 s \\
\midrule
\textbf{Kinodynamically Feasible} &  & \cmark & \cmark & \xmark & \cmark \\
\bottomrule
\end{tabular}
\label{tab:comparison_mpd}
\vspace{-10pt}
\end{table*}

The (instantaneously) safe set~$\calS$ from \eqref{eq:safe_set} for tractor-trailers is formally defined as:
\begin{equation}\label{eq:tt_safe_set}
\begin{aligned}
\calS = \Big\{ & \vx \in \StateSpace \ \Big| \ |\theta^{1}-\theta^{2}| \le \theta_{\textup{max}} \quad \text{(no jackknifing)},\\
& \left(\calR_{\textup{tr}}(\vx) \cup \calR_{\textup{tl}}(\vx)\right) \cap \calO = \emptyset \quad \text{(no collision)} \Big\}.
\end{aligned}
\end{equation}
For the kinematic bicycle model, $\calS$ reduces to collision avoidance with obstacles only.

We compare our proposed method against the following safety strategies commonly used for diffusion planners:

\textbf{(i) Na\"ive Penalty}: This method adds a high penalty term to the cost function $J$~\eqref{eq:optimality_distribution} for any state that lies outside the safe set $\calS$.

\textbf{(ii) Projection}: It defines a projection operator~$\calP_{\calS}$ onto $\calS$ subject to system dynamics:
\begin{equation} \label{eq:projection}
\begin{split}
\raisetag{3.0ex}
\vu_t^{\star} &= \calP_{\calS}(\vu_{\textup{nom}, t};\vx_t) \coloneqq \argmin_{\vu_{t} \in \ControlSpace} \, \|\vu_{t}-\vu_{\textup{nom}, t}\|_{2}^{2} \\ 
\text{s.t.}\quad\quad & \hat{\vx}_{t+1} = f(\vx_t,\vu_t), \quad \hat{\vx}_{t+1} \in \calS.
\end{split}
\end{equation}
This projection~\eqref{eq:projection} is applied recursively to the $K$ trajectory samples $\calY_{1:K}$ during the diffusion process, and to the final output $Y^{(0)}$, instead of shielded rollout in \autoref{alg:safe_mpd}. Projection for diffusion models was introduced in \cite{christopher_constrained_2024} and extended to trajectory planning in \cite{romer_diffusion_2025} under the assumption that the safe set is convex and the dynamics is linear. We will show that its computational overhead becomes intractable for the non-convex and nonlinear tasks in \autoref{subsec:exp_results}.

\textbf{(iii) Guidance}: This method performs gradient descent on a given state trajectory~$\tau_{\vx}^{1} \coloneqq \tau_{\vx}$ to gradually steer it away from the unsafe set:
\begin{equation} \label{eq:guidance_compact}
\tau_{\vx}^{j+1} = \tau_{\vx}^{j} - \texttt{clip}\left(\alpha_{\textup{g}} \nabla \calJ\left( \tau_{\vx}^{j} \right), -\epsilon, \epsilon\right), \, j=1,\ldots,N_{\textup{iter}},
\end{equation}
where $\calJ(\cdot)$ indicates the amount of safety violation:
\begin{equation}
\begin{split}
\raisetag{13.0ex}
& \calJ(\tau_{\vx}) \coloneqq \sum_{t=1}^{\horizon} \Bigg[
\underbrace{\max\big(0,\ |\theta^{1}_{t}-\theta^{2}_{t}| - \theta_{\textup{max}}\big)}_{\text{(1) hitch-angle violation}} \\
& + \underbrace{\sum_{j=1}^{N_{\textup{obs}}} \max\big(0, R_{j} - \mathrm{dist}\big(\calR_{\textup{tr}}(\vx_t)\cup \calR_{\textup{tl}}(\vx_t),\calO_{j}\big)\big)}_{\text{(2) collision violation}} \Bigg].
\end{split}
\end{equation}
The number of guidance steps~$N_{\textup{iter}}$ and the step size~$\alpha_{\textup{g}}$ are set to 3 and 0.05, respectively. The guidance update is clipped to a maximum magnitude $\epsilon$~\cite{zhong_guided_2023}. $R_{j}$ is the safety margin for $j$-th obstacle.\footnote{Constructing a differentiable signed-distance function for our non-convex, articulated geometry requires costly computation, so we employ this more tractable objective.} As with projection, guidance is applied to the $K$ samples during diffusion and to $Y^{(0)}$, replacing shielded rollout in \autoref{alg:safe_mpd}.

All baselines are implemented within the MPD framework with $N=100$, $K=20,000$, $\horizon=50$, and $T_{s}=0.25$~s. The running and terminal costs penalize position error and heading error relative to the goal pose. Heading errors are wrapped to $[-\pi,\pi]$, so both forward and backward parking minimize the heading objective. For tractor-trailers, the positional cost uses the minimum of the tractor and trailer positional errors to the goal, allowing the cost to be minimized regardless of whether the vehicle parks in a forward or backward orientation.

All algorithms are implemented in Python using the JAX library~\cite{jax2018github} to enable GPU-accelerated rollouts. For \textbf{Q3}, we measure the total computation time from the initial noise $Y^{(N)}\sim\calN(\mathbf{0},\mathbf{I})$ to the final optimized trajectory. All experiments were conducted on an NVIDIA RTX 4090 GPU.

%% file: _V.Experiments/d_result.tex
The automated parking environment contains 36 obstacles. The goal position is set to the center of the designated parking slot, with the goal heading perpendicular to the slot's width. Initial states are uniformly sampled from the free space excluding trivial initial poses from which a straight maneuver could solve the task. Each method is evaluated over 100 randomized trials for each model. For hyperparameters, we tune the temperature~$\lambda$~\eqref{eq:optimality_distribution} and running/terminal cost weights on the kinematic bicycle model using Optuna~\cite{akiba_optuna_2019}. For \textbf{Q2}, we then \emph{reuse} these hyperparameters for both tractor-trailer systems. We report three metrics: (i) \textbf{Success Rate}: the percentage of trials where the tractor or trailer footprint enters the goal area without any safety violation; (ii) \textbf{Safety Violations}: the percentage of trials where any constraint (collision or jackknifing) is violated; (iii) \textbf{Computation Time}. Further implementation details are available in our code repository.

The quantitative results are summarized in \autoref{tab:comparison_mpd}. On the kinematic bicycle model, all methods except guidance achieve a 100\% success rate, since this task is comparatively easy and does not involve a trailer. Guidance exhibits safety violations across all systems because it provides no formal safety guarantee. Moreover, because it applies a post-processing correction to the state trajectory, the resulting trajectory may no longer be kinodynamically feasible. For the projection method only, we simplify the environment to contain just 6 circular obstacles near the goal position. Even then, it requires an average of 32.664 minutes for the kinematic bicycle model, although it guarantees safety. For the tractor-trailer systems, it hits the 1-hour timeout in all trials.

When tested on the more complex tractor-trailer systems, the limitations of the baselines become evident. Both the Na\"ive Penalty and Guidance methods show a significant drop in success rate and a sharp increase in safety violations, highlighting their inability to handle the challenges of higher-dimensional, non-convex problems. In contrast, our Safe MPD$^{\shield}$ maintains a perfect 0\% safety violation rate across all models and tasks \textbf{(Q1)}. 
It also achieves near-perfect success rates of 100\% and 98\% for the kinematic and acceleration-controlled tractor-trailer models, respectively (see Fig.~\ref{fig:main_results}(a)), demonstrating its scalability even without model-specific hyperparameter tuning \textbf{(Q2)}. Fig.~\ref{fig:main_results}(b) shows a successful trajectory that includes multi-point turns in a tight space, with no hitch-angle violations or collisions.
Regarding \textbf{Q3}, our method's computational performance is highly competitive. The primary overhead of the shielded rollout comes from the finite-horizon rollouts under the backup policy during validity checks, which are highly parallelizable and can be computed efficiently on a GPU. As a result, Safe MPD$^{\shield}$ achieves computation times comparable to the fastest (but unsafe) baseline on the kinematic models. The higher computation time for the acceleration-controlled model stems from its second-order dynamics. Verifying safety requires a longer horizon~$T_{B}$ with the backup policy~$\pi_{\textup{backup}}$ to ensure the system can be brought to a full stop within $\calC$. This sensitivity correlates with tighter input bounds: stricter input limits require more recovery steps to verify validity.

\begin{figure*}[t]
\centering
\includegraphics[width=0.93\linewidth]{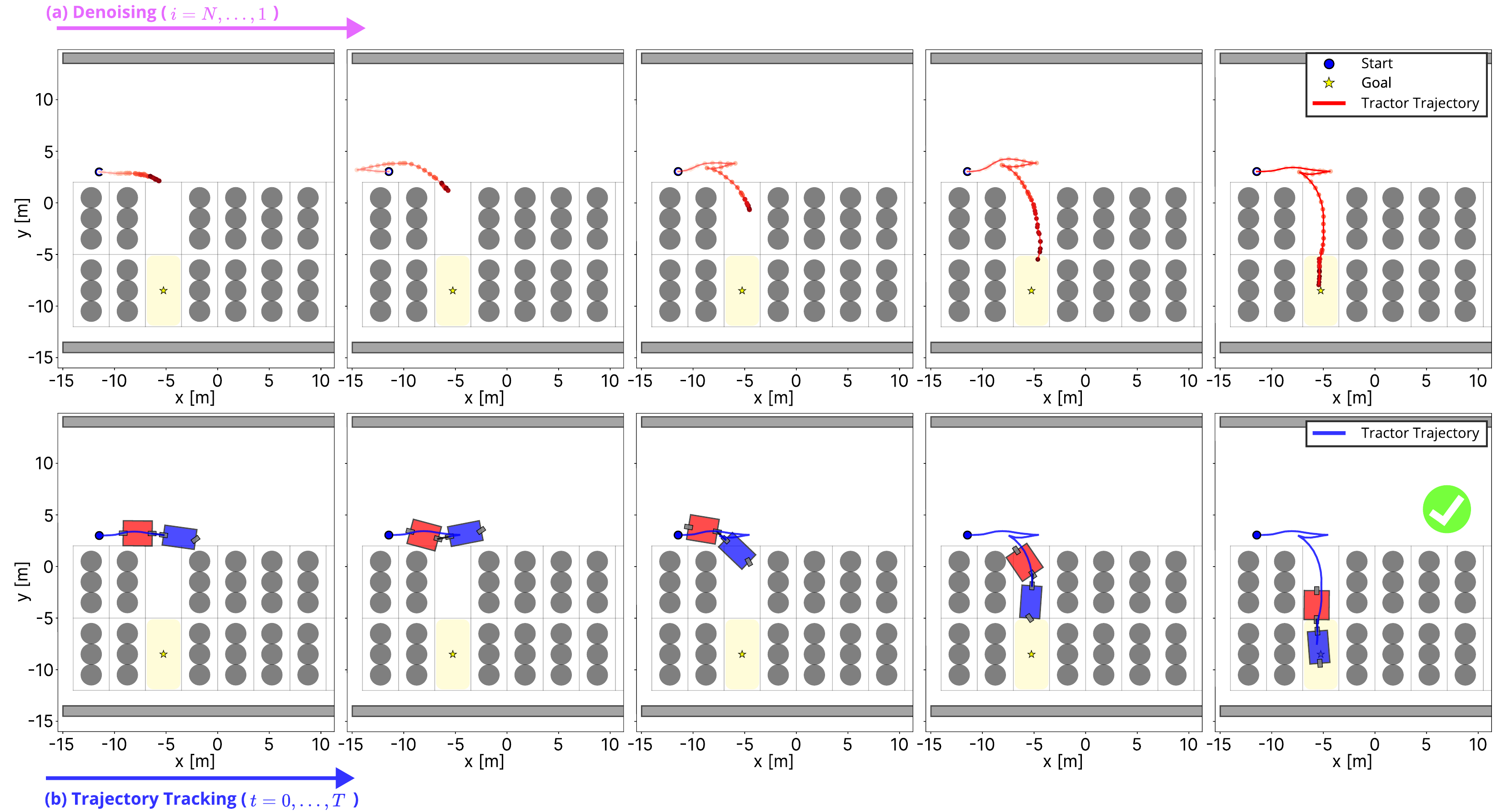}
\caption{Visualization of the diffusion process and final trajectory execution. (a) Snapshots of the diffusion trajectory for a kinematic tractor-trailer at different denoising steps, showing the refinement from a random, high-cost path (e.g., $i=100$) to an optimized solution at the final step ($i=0$). (b) Final trajectory with the tractor–trailer footprint rendered along the path; despite tight clearances, neither hitch-angle nor collision constraints are violated.
}
\label{fig:main_results}
\vspace{-10pt}
\end{figure*}

\textbf{Integration test with a tractor-trailer navigation stack.} For \textbf{Q4}, we integrate our algorithm into our existing tractor-trailer navigation framework~\cite{majd_gpu-accelerated_2025}, replacing the Hybrid A* planner with Safe MPD$^{\shield}$. The time to generate a feasible path drops from several minutes to under a second. Because our method generates kinodynamically feasible trajectories, the framework's downstream tracking controller (BR-MPPI~\cite{parwana_br-mppi_2025}) reliably tracks the resulting diffusion trajectories with multi-point turns (see our project page).


%% file: _VI.Conclusion/conclusion.tex
In this work, we introduced Safe Model Predictive Diffusion (Safe MPD$^{\shield}$), a novel framework for trajectory optimization that integrates a formal safety shield directly into the denoising process of a model-based diffusion planner. Our method demonstrates three advantages that are critical for real-world robotics: the generated trajectories are (i) \emph{kinodynamically feasible by construction}, (ii) \emph{provably safe}, and (iii) \emph{computationally efficient} through batched rollouts and Monte Carlo score ascent on a GPU. The strong performance and sub-second planning times on complex, non-convex trajectory-planning problems, such as the tractor-trailer parking task, highlight the potential of Safe MPD$^{\shield}$ to become a powerful tool for real-world autonomous systems. Our future work will focus on the deployment and validation of this framework on physical hardware.